\documentclass{article}

\usepackage[preprint,eandd,nonatbib]{neurips_2026}
\usepackage[utf8]{inputenc}
\usepackage[T1]{fontenc}
\usepackage{courier}
\usepackage{microtype}
\usepackage{amsmath,amssymb,amsfonts,amsthm}
\usepackage{booktabs}
\usepackage{array}
\usepackage{graphicx}
\usepackage{xcolor}
\usepackage{enumitem}
\usepackage{hyperref}
\usepackage[nameinlink,noabbrev]{cleveref}
\usepackage{caption}
\usepackage{subcaption}

\hypersetup{
  colorlinks=true,
  linkcolor=blue!60!black,
  citecolor=blue!60!black,
  urlcolor=blue!60!black,
}

\newtheorem{theorem}{Theorem}
\newtheorem{proposition}[theorem]{Proposition}
\newtheorem{definition}[theorem]{Definition}

\theoremstyle{remark}

\theoremstyle{plain}

\newcommand{\method}{\textsc{MemAudit}}
\newcommand{\compiler}{\textsc{GVT}}
\newcommand{\budget}{B}
\newcommand{\experiences}{\mathcal{E}}
\newcommand{\ground}{\mathcal{U}}

\newcommand{\units}{\mathcal{R}}
\newcommand{\feasible}{\mathcal{F}}
\newcommand{\package}{\mathcal{P}}
\newcommand{\opt}{\mathrm{OPT}}

\title{\method:\\
An Exact Package-Oracle Evaluation Protocol\\
for Budgeted Long-Term LLM Memory Writing}

\author{%
Nishant Bhargava\\
Purdue University\\
\texttt{bharga57@purdue.edu}
\And
Rodrigo Sobral Barrento\\
Purdue University\\
\texttt{rbarrent@purdue.edu}
}
\date{}

\begin{document}
\maketitle

\begin{abstract}
Long-term LLM agents must compress streams of past interactions into persistent memory before future queries are known. Existing evaluations usually measure final question-answering accuracy, which entangles memory writing with retrieval, prompting, and reader reasoning. We introduce \method, an exact package-oracle evaluation protocol for budgeted long-term memory writing. A \method\ package fixes an experience stream, candidate memory representations, storage costs, semantic evidence units, future-query requirements, and a budget, turning write-time memory selection into a finite auditable optimization problem with a certified denominator. We instantiate this protocol with a concave-over-modular semantic coverage objective under storage and one-representation-per-experience constraints, and compute exact package optima using branch-and-bound with MILP certification. Across controlled exact packages, validity-heavy stress tests, human-audited natural support slices, and exported Mem0, A-Mem, and Letta stores, \method\ separates representation quality, validity-state preservation, and budget-aware selection effects that end-to-end QA cannot localize. The resulting artifact provides reusable package generators, certified solvers, natural package exports, external-system scorers, and cached reproducibility metadata for evaluating what memory writers actually preserve under fixed storage budgets.
\end{abstract}

\section{Introduction}

LLM agents increasingly operate across sessions: they converse with users, call tools, edit code, inspect documents, and later need to reuse what happened. A fixed context window makes the naive policy of retaining everything impossible. The agent must compile an experience stream into persistent memory, choosing not only whether to store an item but also which representation should survive: raw span, extracted fact, temporal event, graph edge, summary, rule, skill, compound update, or tombstone.

This paper argues that long-term memory writing should be evaluated as finite semantic compression, not only as an architecture choice. The load-bearing object is a package with an exact package oracle: for a fixed budget, finite candidate memories, and fixed future-query requirements, how close is a writer to the best package-feasible store? This separates write quality from downstream retrieval and reader behavior. A system may fail because it wrote the wrong memory, because retrieval missed a good memory, or because the reader ignored evidence; final QA accuracy alone cannot identify which layer failed.

The distinction matters because most deployed memory systems are evaluated only after several other choices have intervened. A memory writer may extract a useful fact but store it in a form too expensive to keep under the target budget. A retriever may miss the relevant memory even if the writer preserved it. A reader may see the evidence and still answer incorrectly. These are different engineering failures. \method\ isolates the first one by freezing a finite candidate set, a storage accounting rule, and an evidence-unit objective. The resulting denominator is not a claim about all possible memories; it is a reproducible answer to the narrower question ``given this package, how much of the package-feasible semantic value did the writer preserve?''

\begin{figure}[t]
\centering
\includegraphics[width=\linewidth,trim=0 24 0 0,clip]{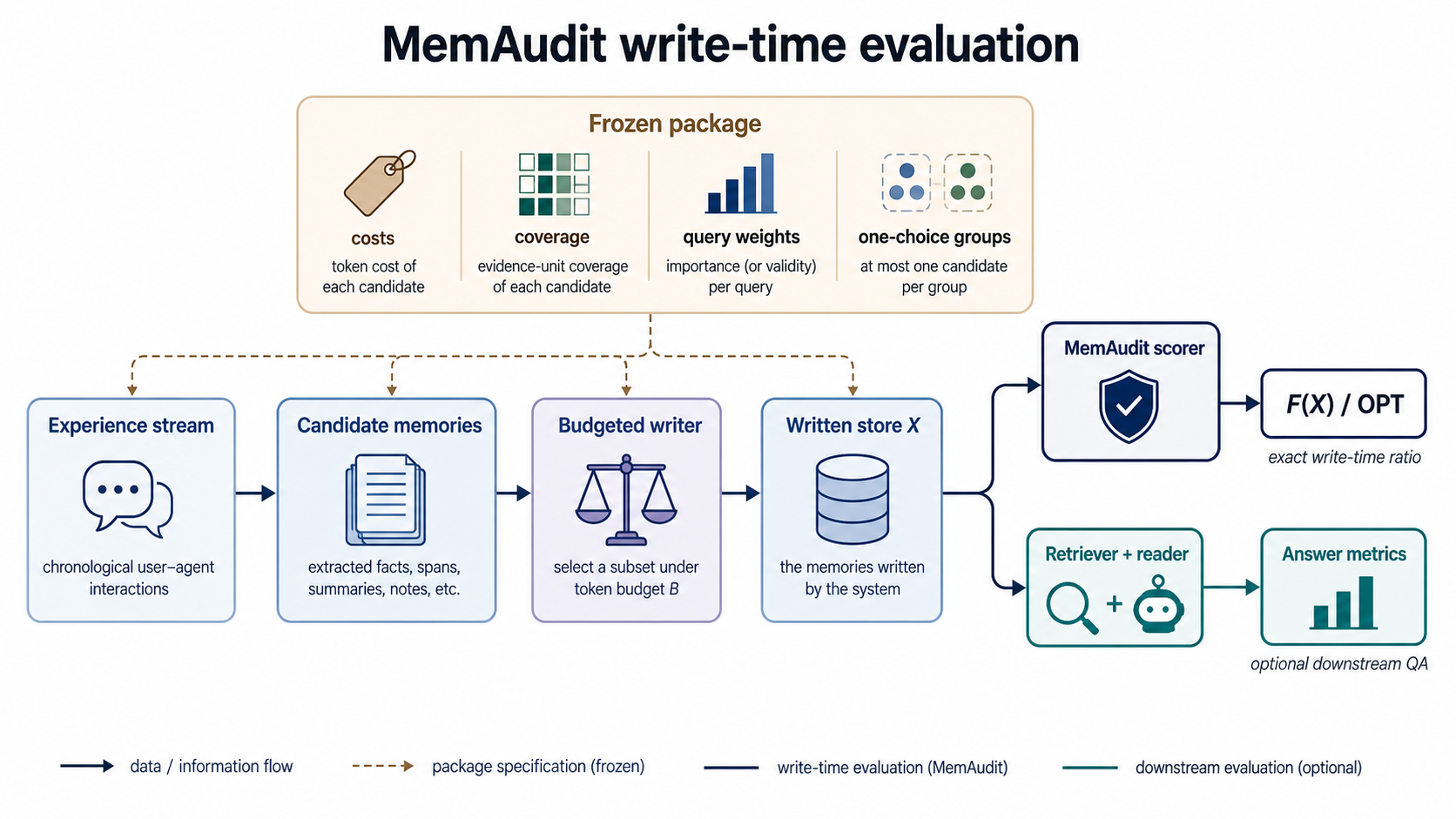}
\caption{\method\ evaluates the written memory store before retrieval and reader reasoning. The frozen package defines the exact write-time denominator, while downstream QA is measured separately through the optional retriever/reader branch.}
\label{fig:pipeline}
\end{figure}

\paragraph{Contributions.}
We make four evaluation-and-dataset contributions. (i) We define a finite \method\ package for budgeted memory writing with a package-oracle ratio and a union denominator for external stores. (ii) We instantiate the known concave-over-modular submodular coverage structure for semantic memory preservation. (iii) We provide exact/certified package optima and a package-oracle reference baseline, with an independent MILP audit. (iv) We release an initial \method\ suite spanning controlled exact packages, validity-heavy stress packages, human-validated model-adjudicated natural support slices with reported inter-annotator agreement, and exported-system packages for Mem0, A-Mem, and Letta.

\paragraph{E\&D artifact role.}
\method\ is designed to complement systems such as Mem0, Letta, A-Mem, MemGPT-style archival memory, and graph-based memories by scoring their exported writes under a shared denominator. The released artifact contains deterministic exact-package generators, certified solvers, cached natural package exports, adjudication summaries, exported Mem0/A-Mem/Letta memory stores, and scripts to reproduce the main tables and figures without additional API calls. API-backed construction is needed only to regenerate natural annotations or rerun external memory exports.

\paragraph{Relation to MemSim.}
MemSim/MemDaily \cite{zhang2024memsim} automatically constructs reliable QA pairs for evaluating personal-assistant memory through downstream answering. \method\ targets a complementary layer: it evaluates the write-time memory store itself by defining finite candidate representations, costs, evidence units, and an exact budgeted optimum. Thus MemSim asks whether an agent answers generated memory questions correctly, while \method\ asks how close a written memory store is to the best package-feasible semantic store before retrieval and reader reasoning are invoked. This also lets \method\ localize where the problem sits: candidate generation, representation choice, budget-aware selection, retrieval, or reader use.

\section{\method\ Package}\label{sec:package}

An LLM agent observes experiences $\experiences=(e_1,\ldots,e_T)$. Each experience can generate multiple candidate write choices, such as raw text, an atomic fact, a summary, a graph edge, a rule, a compound update, or a tombstone. Let
\[
\ground=\{(i,j): i\in[T],\ j\in J_i\setminus\{\mathrm{discard}\}\}
\]
be the virtual ground set of non-discard experience-representation choices. Each element $u=(i,j)$ has cost $c_u>0$, and $G_i=\{(i,j):j\in J_i\setminus\{\mathrm{discard}\}\}$ is the group of choices for experience $i$. A feasible memory store satisfies one storage budget and one choice per experience:
\begin{equation}
\feasible_\budget=
\left\{X\subseteq\ground:
\sum_{u\in X}c_u\leq\budget,\quad |X\cap G_i|\leq1\ \forall i
\right\}.
\label{eq:feasible}
\end{equation}
Thus the feasibility structure is the intersection of one knapsack constraint and one partition matroid.

This virtual-ground-set view makes representation choice explicit instead of treating memory as a homogeneous list of strings. For example, suppose a user first says they prefer vegetarian travel meals and later says they are pescatarian now. A future query asks what meal should be booked. The package may contain a stale vegetarian fact, a current pescatarian fact, a raw span containing both turns, a tombstone invalidating the old preference, and a compound update saying ``vegetarian is superseded by pescatarian.'' These candidates have different costs and cover different evidence units: current truth, invalidation, temporal order, or raw provenance. The constraint $|X\cap G_i|\leq1$ prevents the package optimum from keeping every representation of one experience and makes the denominator comparable to a writer that chooses one persistent form. Because some deployed systems keep raw+summary+fact variants, the artifact also supports a relaxed $|X\cap G_i|\leq k$ audit; on Natural-87, allowing $k=2$ changes exact OPT by only $1.3$--$6.0\%$ across the evaluated budgets (\Cref{app:sensitivity}).

The package also fixes storage accounting. In our experiments, the default cost is normalized word-equivalent storage over the serialized memory exposed to the writer/reader. This is a declared benchmark cost, not a universal systems-cost model: real systems also pay embedding-index slots, graph-node overhead, JSON metadata, and retrieval compute. To check that the main diagnostics are not an artifact of whitespace word counts, \Cref{app:sensitivity} recomputes the Natural-87 package and exported-store pruning under an alternative cost $c(u)=8+\lceil\mathrm{bytes}(u)/24\rceil$, approximating fixed per-record overhead plus serialized payload. The qualitative exported-system ordering is stable under this rule. When external systems export memories outside the package candidate set, those memories are added to a union package with their own measured costs rather than forced into the original candidate taxonomy.

The benchmark defines semantic evidence units $\units=\{r_1,\ldots,r_M\}$. A future query $q$ has required evidence units $R(q)\subseteq\units$ and nonnegative importance weights, inducing evidence weights $w_r\geq0$. Each candidate memory has a nonnegative coverage row $a_{ur}\in[0,1]$. The utility of a store is
\begin{equation}
F(X)=\sum_{r\in\units} w_r\,
h_r\!\left(\sum_{u\in X}a_{ur}\right),
\label{eq:coverage}
\end{equation}
where each $h_r$ is concave, nondecreasing, and normalized. The default package objective uses $h_r(z)=\min(1,z)$, so duplicate memories of the same evidence unit have diminishing marginal value.

Evidence units are the semantic atoms used by the benchmark. In a synthetic package they are generated from the hidden event graph; in a natural support slice they are source-backed units extracted from the support sessions and then mapped to future-query requirements. A unit can represent a fact, a temporal relation, an entity preference, a deletion, an abstention condition, or a validity-state update. The objective is positive coverage: stale-fact avoidance is represented by covering the evidence unit that says an older fact is no longer current, not by assigning negative utility to stale memories. This keeps the objective monotone while still allowing the benchmark to ask whether writers preserve current-truth information.

The concavity in \Cref{eq:coverage} encodes diminishing returns. If two candidates cover the same evidence unit, the second copy should usually help less than the first. If a query requires multiple evidence units, the weights $w_r$ distribute value across those requirements. The objective is therefore a semantic surrogate for write-time preservation, not a replacement for downstream QA. A reader can still fail with a high-$F$ store, and a low-$F$ store might answer an easy question by chance. The point is that $F$ gives a deterministic, package-local target for the writer layer.

\begin{theorem}[Semantic coverage is monotone submodular]\label{thm:coverage}
Let $w_r\geq0$, $a_{ur}\geq0$, and let each $h_r$ be concave, nondecreasing, and satisfy $h_r(0)=0$. Then $F$ in \Cref{eq:coverage} is normalized, monotone nondecreasing, and submodular on $2^\ground$.
\end{theorem}
The proof is the standard concave-over-modular diminishing-returns argument and is given in \Cref{app:proofs}. The theorem supports the semantic surrogate; it does not claim black-box LLM answer accuracy is submodular.

\begin{definition}[\method\ package and package ratio]\label{def:package}
An \method\ package is a tuple
\[
\package=(\ground,\mathcal{G},c,\units,A,w,\budget).
\]
Its exact package optimum and package ratio are
\[
\opt_{\package}(\budget)=\max_{X\in\feasible_\budget(\package)}F_{\package}(X),
\qquad
\rho_{\package}(X)=F_{\package}(X)/\opt_{\package}(\budget).
\]
\end{definition}

\begin{definition}[Union denominator for external stores]\label{def:union}
If an external memory system writes memories $Y$ not contained in $\ground$, we evaluate them in the finite union package $\package^+(Y)$ obtained by adding $Y$, their costs, and their adjudicated coverage rows. The external-store ratio is
\[
\rho_{\mathrm{union}}(Y)=
\frac{F_{\package^+(Y)}(Y)}{\opt_{\package^+(Y)}(\budget)}.
\]
We also report an analysis-only upper-pruned bound over subsets of $Y$ to separate extraction quality from budget-aware selection.
\end{definition}

The union denominator is essential for scoring real systems. Suppose Mem0 writes a memory that is not one of our package candidates. Scoring it only against the package candidate optimum would be ambiguous: the system may have created a useful representation that the package did not contain. In $\package^+(Y)$, the exported memory becomes a first-class candidate with an adjudicated coverage row. The numerator scores exactly what the system wrote and retained, while the denominator asks what the best budget-feasible subset could have achieved using both package candidates and system exports. This makes low scores interpretable: a low raw exported-store ratio with a high upper-pruned bound indicates that the system extracted useful content but did not select or compact it well under budget.

\section{Exact Optima and Reference Writers}\label{sec:oracles}

Exact optimization is central to \method. For small packages, we compute $\opt_{\package}(\budget)$ by branch-and-bound over experience-representation assignments. For the default clipped-coverage objective, we also use a MILP certificate with binary candidate variables $x_u$ and coverage variables $y_r\leq\sum_u a_{ur}x_u$, $y_r\leq1$. Greedy or learned references are never labeled as OPT.

The branch-and-bound solver searches over groups rather than over unconstrained candidate subsets. At each group it branches over discard or one representation choice, tracks remaining budget, and uses an optimistic fractional bound over current marginal gains to prune subtrees. The bound is admissible for the clipped coverage objective because future marginal coverage can only decrease as more evidence units are covered. This solver is sufficient for the exact-small and adjudicated natural packages we report. The MILP audit is included to reduce the risk that an implementation detail in the custom solver defines the benchmark.

\begin{table}[h]
\centering
\small
\caption{Exact-solver certification. PuLP MILP and pure-Python branch-and-bound were run on the same requested-scope audit instances; equality is objective-value equality, allowing tied optimal stores with different candidate ids.}
\label{tab:milp}
\begin{tabular}{lccc}
\toprule
Audit & Rows & Objective matches & Max diff \\
\midrule
B\&B vs MILP & $1{,}200$ & $1{,}200/1{,}200$ & $0.0$ \\
\bottomrule
\end{tabular}
\end{table}

The package-oracle reference baseline is grouped value-threshold (\compiler). For each arriving experience, \compiler\ forms the budget-feasible candidates whose package-oracle marginal density exceeds a threshold, then inserts the admissible representation with largest raw marginal value. A threshold grid gives a conservative insertion-only constant-factor guarantee under exact marginals and small-item assumptions; the theorem and proof are in \Cref{app:proofs}. \compiler\ is included as a calibration baseline: because it has access to exact package marginals, it helps verify that the package and solver behave sensibly. It is not the proposed deployed writer.

Density-only representation choice is included as a negative control. It can be arbitrarily bad even for modular utility: a tiny candidate can have higher density while losing nearly all value. Controlled experiments below show this failure empirically.

This calibration role is useful because it makes representation-choice effects visible. A pure density rule may prefer a cheap but narrow memory; a pure value rule may spend the budget too quickly on expensive raw spans. \compiler\ combines a density threshold with within-group value choice, which is well aligned with the one-representation-per-experience structure, while remaining secondary to the benchmark artifact itself.

\begin{proposition}[Density-only can be arbitrarily bad]\label{prop:density-bad}
For every $\eta>0$, there is a one-experience modular instance satisfying $0<c_u\leq\budget/2$ in which a max-density representation rule obtains at most $\eta$ times the finite package optimum.
\end{proposition}

\begin{theorem}[Scoped package-oracle \compiler\ guarantee]\label{thm:gvt}
Let $F$ be normalized, monotone, and submodular on $\ground$. Suppose candidates arrive group-by-group, the algorithm is insertion-only, computes exact package-oracle marginals, enforces \Cref{eq:feasible}, and chooses the maximum raw marginal among threshold-admissible candidates in each arriving group. If every non-discard representation has cost $0<c_u\leq\budget/2$ and the threshold grid contains
\[
\lambda\in\left[(1-\varepsilon)\frac{\opt}{2\budget},\,\frac{\opt}{2\budget}\right],
\]
then the best grid solution satisfies $F(X)\geq(1-\varepsilon)\opt/4$. This is a calibration result for the package-oracle baseline; deployed writers without exact marginals are evaluated empirically.
\end{theorem}

\section{Controlled Exact-OPT Packages}\label{sec:controlled}

\method-Small generates compact hidden event graphs with raw, fact, summary, tombstone, and compound-update candidates. The exact 500-seed budget sweep maps fractional budgets $0.01,0.02,0.05,0.10,0.20$ to absolute budgets $B=1,2,4,8,16$; we discuss the non-degenerate budgets $B=2,4,8,16$.

Each package instance contains enough redundancy to make representation choice nontrivial. Raw spans are broad but expensive; atomic facts are compact but may miss temporal or validity relations; summaries can cover several units at intermediate cost; tombstones and compound updates encode current-state transitions that are not recoverable from a stale fact alone. Baselines are constrained to the same package and budget. Recency raw keeps recent raw memories until the budget is exhausted. Fact-only and summary-only restrict the candidate set to one representation family. Density-only chooses the candidate with highest current marginal value per unit cost. No-tombstone OPT removes validity candidates and then resolves the exact problem, measuring the frontier loss from omitting that representation class.

\begin{figure}[t]
\centering
\includegraphics[width=0.72\linewidth]{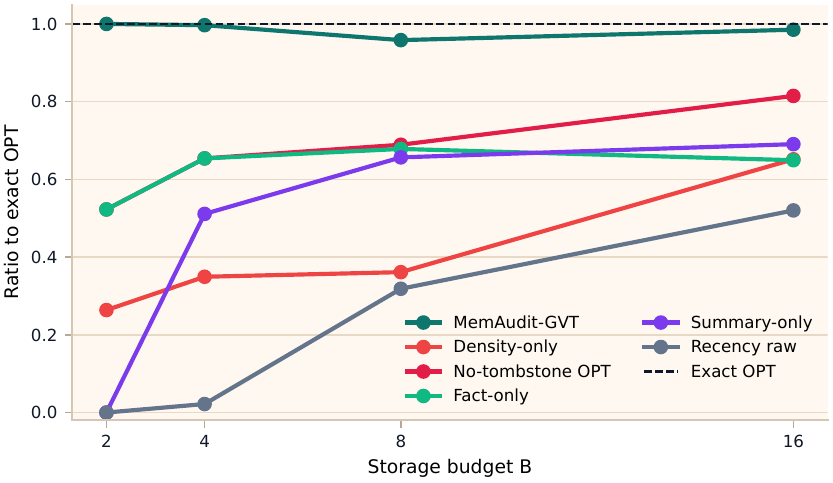}
\caption{Exact-small budget sweep. \method-\compiler\ remains near certified package OPT, while density-only, recency raw, and single-representation baselines degrade under tight storage. Bands are bootstrap 95\% intervals from the canonical 500-seed run.}
\label{fig:exact}
\end{figure}

\Cref{fig:exact} shows that the exact package denominator separates memory-writing policies. \method-\compiler\ covers all invalidation units selected by OPT at the reported budgets (mean invalidation coverage $2.00$ for each of $B=2,4,8,16$). Density-only covers no invalidation units at $B=2,4,8$ and reaches only $0.264$--$0.361$ of OPT there. Tombstone/compound-update candidates expand the feasible frontier: exact no-tombstone OPT reaches $0.523,0.654,0.689,0.815$ of full OPT across $B=2,4,8,16$.

The exact sweep is the cleanest evidence that \method\ measures the intended layer. The denominator is certified, all methods face the same candidate set and budget, and the gaps are not downstream retrieval artifacts. The strongest negative result is not that one heuristic loses on one seed, but that a plausible local rule systematically discards the validity-state units that exact OPT preserves. The no-tombstone OPT curve is also important because it is not an algorithm comparison: even the best package-feasible store is worse when the candidate generator lacks current-state representations. This distinguishes candidate-generation bottlenecks from selection bottlenecks.

\section{Validity-State Stress Packages}\label{sec:validity}

Memory systems often fail by preserving stale facts. We encode current truth, supersession, deletion, and abstention as positive validity-state evidence units rather than negative utility. A tombstone therefore covers an invalidation unit such as ``the old address is no longer current,'' while stale retrieval and stale answers remain separate diagnostic metrics.

This design choice is what lets validity fit into a monotone objective. A memory store is rewarded for preserving the evidence needed to answer ``what is true now?'' or ``should the agent abstain because the previous answer is invalid?'' It is not directly penalized for retaining old memories. A stale raw span may still be useful if paired with a later tombstone or compound update, because the reader can infer the update sequence. Without such validity-state candidates, a writer can spend the budget on high-density facts and still fail on future queries that ask for current truth.

\begin{figure}[t]
\centering
\begin{subfigure}{0.52\linewidth}
\centering
\includegraphics[width=\linewidth]{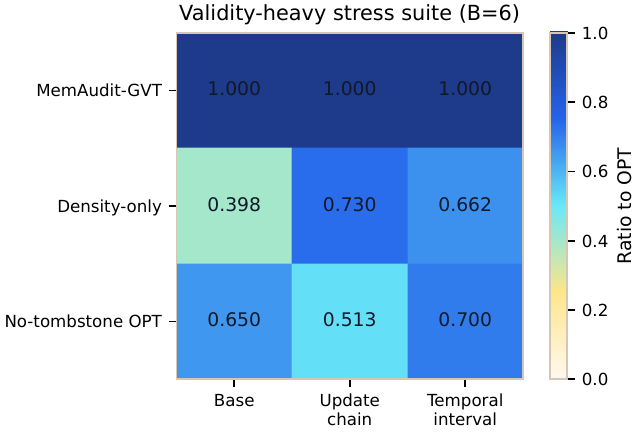}
\caption{Ratio to OPT at $B=6$.}
\end{subfigure}
\hfill
\begin{subfigure}{0.42\linewidth}
\centering
\includegraphics[width=\linewidth]{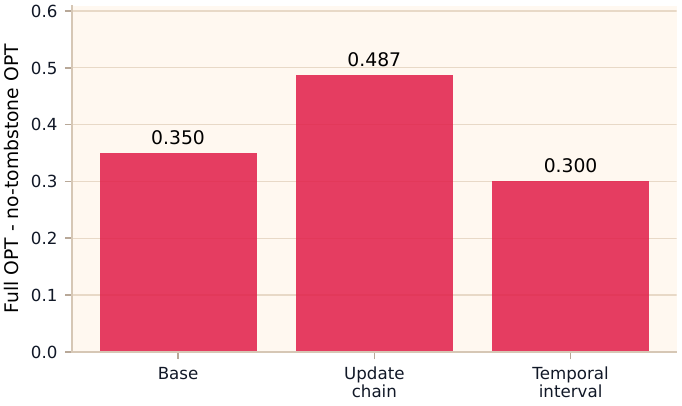}
\caption{Full OPT minus no-tombstone OPT.}
\end{subfigure}
\caption{Validity-heavy stress visualization for the headline exact distributions: base, update-chain, and temporal-interval.}
\label{fig:stress}
\end{figure}

The stress suite sharpens the validity-memory claim. At $B=6$, no-tombstone OPT falls to $0.513$ on \texttt{update\_chain} and $0.700$ on \texttt{temporal\_interval}, so validity-state candidates expand the package-feasible frontier in controlled update-heavy instances. Other stressors are useful diagnostics, but the main paper restricts claims to the distributions in \Cref{fig:stress}.

The three headline distributions cover different validity pressures. The base distribution has ordinary preferences and events with occasional updates. The \texttt{update\_chain} distribution creates repeated supersession, where the useful memory is often the latest current-state representation plus evidence that older states are obsolete. The \texttt{temporal\_interval} distribution stresses interval boundaries and recency of truth, so compact temporal updates matter more than raw recency alone. The full stress suite also includes density traps, scope shifts, redundancy-heavy packages, and hard abstention cases, but those are deferred because the central main-text claim is narrower: validity-state representation choices can expand the package-feasible semantic frontier.

\section{Natural and Exported-System Packages}\label{sec:natural}

Natural packages demonstrate that the same finite-package schema can be populated from realistic conversation support slices. We built a 200-example LongMemEval-style support-slice package with Gemini Flash-Lite through OpenRouter. Candidate generation sees answer-support sessions; query and gold-answer fields are used only in the separate benchmark-labeling pass that maps evidence units to query requirements. The package is support-sliced, not full-haystack LongMemEval, and the labels are source-backed model-adjudicated labels with explicit rejection and audit paths.

The natural construction route is intentionally auditable. The pipeline separates three artifacts: candidate memories, evidence units, and query requirements. Candidate memories are generated from support sessions. Evidence units are grounded in the same source text. Query requirements are assigned in a separate labeling pass. This separation reduces direct answer leakage into the writer while still allowing exact package scoring after the artifact is frozen. The resulting Natural-87 subset is deliberately modest: it is small, support-sliced, excludes 13 ambiguous examples from the stricter 100-example adjudication pass, and still depends on model-adjudicated evidence-unit construction. We therefore use it as a first audited natural package and exported-system diagnostic, not as a final full-haystack long-memory benchmark.

\begin{table}[h]
\centering
\scriptsize
\setlength{\tabcolsep}{2.5pt}
\renewcommand{\arraystretch}{1.08}
\caption{Natural package construction reliability. The package format exposes which examples are schema-valid, adjudicated for scoring, or flagged for ambiguity instead of silently mixing them into one benchmark pool.}
\label{tab:reliability}
\begin{tabular}{@{}>{\raggedright\arraybackslash}p{0.20\linewidth}
>{\raggedright\arraybackslash}p{0.10\linewidth}
>{\raggedright\arraybackslash}p{0.25\linewidth}
>{\raggedright\arraybackslash}p{0.09\linewidth}
>{\raggedright\arraybackslash}p{0.22\linewidth}@{}}
\toprule
Artifact & Scope & Resolved/agreement & Disputed & Role \\
\midrule
Natural-200 primary & $200$ ex. & $200$ schema-valid & -- & scalable package construction \\
Adjudicated subset & $100$ ex. & $87$ accepted/corrected & $13$ & main natural scoring set \\
Flash-Lite spot-check & $30$ ex. & $29$ accepted/corrected & $1$ & independent consistency check \\
Human-edited seed set & $100$ ex. & $100$ schema-valid & -- & prompt/schema development \\
\bottomrule
\end{tabular}
\end{table}

\Cref{tab:reliability} is intentionally conservative. The primary Natural-200 export used 681 API calls, 1{,}978{,}327 tokens, and about \$0.907. An initial secondary audit surfaced ambiguous required-unit cases, motivating a stricter Gemini-Flash adjudication pass on 100 examples; it accepted or corrected 87 examples and rejected 13 for unresolved ambiguity. The independent 30-example spot-check accepted/corrected 29 examples. Rejections mostly involved underspecified query requirements, support spans that licensed multiple plausible evidence decompositions, or candidate memories whose coverage depended on borderline summary entailment. This is an advantage of the package format: ambiguous evidence requirements become measurable dataset objects rather than hidden benchmark noise.

\begin{table}[t]
\centering
\scriptsize
\setlength{\tabcolsep}{2.7pt}
\renewcommand{\arraystretch}{1.08}
\caption{Coverage-label audit summary on the 87-example adjudicated subset. Cell-level rows use the released human audit frame; package-pass rows compare model package adjudications on the same cells. Ratio MAD is the mean absolute ratio change at $B=30,60,100$ where a package-level re-score is defined.}
\label{tab:coverage-audit}
\begin{tabular}{@{}>{\raggedright\arraybackslash}p{0.23\linewidth}
>{\raggedright\arraybackslash}p{0.18\linewidth}
>{\centering\arraybackslash}p{0.10\linewidth}
>{\centering\arraybackslash}p{0.10\linewidth}
>{\centering\arraybackslash}p{0.15\linewidth}
>{\raggedright\arraybackslash}p{0.16\linewidth}@{}}
\toprule
Audit pair & Scope & Cohen's $\kappa$ & Raw agree. & Ratio MAD & Ranking \\
\midrule
Human--human & $1{,}975$ cells & $0.831$ & $0.918$ & -- & -- \\
Human--Gemini & $1{,}813$ agreed cells & $0.863$ & $0.934$ & $0.035$ & preserved \\
Sonnet--Gemini & package pass & $0.649$ & $0.845$ & \begin{tabular}[c]{@{}c@{}}$0.079/0.073$\\$0.070$\end{tabular} & preserved, $\rho{=}1.0$ \\
Sonnet--human & package pass & $0.588$ & $0.811$ & -- & preserved \\
Sonnet--Gemini & blind cell pass & $0.531$ & $0.763$ & -- & -- \\
Sonnet--human & blind cell pass & $0.484$ & $0.739$ & -- & -- \\
\bottomrule
\end{tabular}
\end{table}

The coverage audit in \Cref{tab:coverage-audit} is part of the artifact, not a side note. On the 87-example adjudicated subset, two human annotators labeled all nonzero coverage cells plus a stratified 20\% sample of model-zero cells. Human agreement is substantial, Gemini agrees closely with the human-consensus subset, and re-scoring package writers with human labels changes ratios by only $0.035$ on average while preserving rankings. All natural-package rows have bootstrap intervals in the artifact; for example, Estimated-GVT at $B=100$ is $0.829$ $[0.761,0.887]$ under Gemini labels and $0.828$ $[0.755,0.893]$ under human-consensus labels. The Sonnet rows are useful as a stress test: rankings remain stable, but binary labels are not interchangeable with the human-validated Gemini labels.

The next question is whether the same package machinery can score memory systems that write in very different forms. The union denominator makes this possible: exported memories are added as candidates with measured costs and adjudicated coverage rows, so the score can distinguish extraction quality from budget-aware selection. These exported-system scores are post-hoc budget-pruned diagnostics over written memories, not claims that Mem0, Letta, or A-Mem were natively deployed with our package budget or objective.

\begin{figure}[t]
\centering
\includegraphics[width=0.96\linewidth]{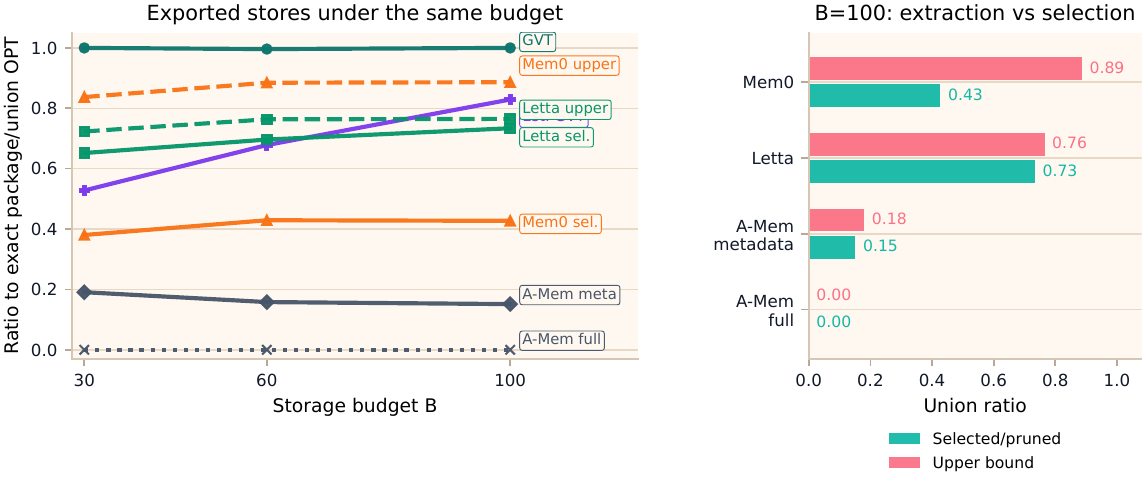}
\caption{Actual-system diagnostic on the 87-example adjudicated subset. Left: package or union ratio across budgets for package-oracle \method-\compiler, non-oracle Estimated-GVT, and exported Mem0/Letta/A-Mem stores. Right: at $B=100$, upper-pruned rows isolate extraction potential, while selected/pruned rows show the budget-aware selection actually achieved by each exported store.}
\label{fig:systems}
\end{figure}

\begin{table}[t]
\centering
\scriptsize
\setlength{\tabcolsep}{3.5pt}
\renewcommand{\arraystretch}{1.08}
\caption{External-system denominator comparison at $B=100$ on Natural-87. $F(Y)$ is the raw mean coverage value of the post-hoc pruned exported store. The package ratio uses the original package-candidate denominator; the union ratio uses the expanded denominator after adding the system's own exported memories.}
\label{tab:external-denominators}
\begin{tabular}{@{}lccc@{}}
\toprule
Exported diagnostic row & Raw $F(Y)$ & Package ratio & Union ratio \\
\midrule
Mem0 salience-pruned & $0.655$ & $0.489$ & $0.427$ \\
Mem0 upper-pruned & $1.443$ & $1.041$ & $0.886$ \\
Letta salience-pruned & $1.270$ & $0.881$ & $0.734$ \\
Letta upper-pruned & $1.322$ & $0.918$ & $0.765$ \\
A-Mem metadata view & $0.282$ & $0.218$ & $0.180$ \\
A-Mem full-store view & $0.000$ & $0.000$ & $0.000$ \\
\bottomrule
\end{tabular}
\end{table}

\Cref{fig:systems} and \Cref{tab:external-denominators} show that \method\ can diagnose heterogeneous memory systems without requiring a common internal representation. The non-oracle Estimated-GVT line is the more realistic package writer: it uses estimated values rather than hidden package-oracle marginals and reaches $0.53,0.68,0.83$ of package OPT across $B=30,60,100$. Letta's core+archival exports achieve the strongest salience-pruned union ratios on the adjudicated subset. Mem0's high upper-pruned bound shows that useful evidence is often extracted, while the gap to salience pruning isolates budget-aware selection as the bottleneck. A-Mem preserves evidence in rich notes, but their serialized size makes them poorly matched to strict small-budget stores unless paired with compaction or admission control.

The denominator comparison makes the union-package point concrete. Mem0's upper-pruned package ratio exceeds one because its exports add coverage not available to the original package candidates; the union denominator is the fairer finite comparison once those memories are admitted as candidate writes. Conversely, scoring only the system's own memories would remove the optimum and obscure whether the bottleneck is extraction or budget-aware selection. The exactly zero A-Mem full-store row is a cost-model diagnostic rather than a coverage bug: at $B=100$, no full serialized A-Mem note fits the budget (mean minimum full-note cost $1768$ words, or $464$ units under the byte-overhead rule), while its compact metadata view remains feasible.

The diagnostic also avoids a common evaluation shortcut. It would be easy to ask each memory system to answer downstream questions and report the final score. That number is useful, but it does not say whether a low score came from writing, retrieval, prompt formatting, model capability, or answer normalization. In \Cref{fig:systems}, the same evidence objective, costs, and budgets are used for all exported stores. This makes the failure mode more actionable: improve extraction if the upper-pruned bound is low; improve admission or compaction if the bound is high but the selected-store ratio is low; improve retrieval/reader behavior only after the written store is known to contain the necessary evidence.

\section{Related Work}\label{sec:related}

\paragraph{Memory systems and benchmarks.}
MemGPT, Reflexion, MemoryBank, Mem0, and A-Mem propose memory architectures or policies for long-term agents \cite{packer2023memgpt,shinn2023reflexion,zhong2023memorybank,chhikara2025mem0,xu2025amem}. LoCoMo and LongMemEval evaluate long-horizon memory behavior \cite{maharana2024locomo,wu2024longmemeval}, while MemSim/MemDaily constructs reliable personal-memory QA data through Bayesian simulation \cite{zhang2024memsim}. \method\ is complementary: it scores exported writes under a package or union denominator before retrieval and reader errors are mixed in, without requiring systems to share an internal memory representation.

\paragraph{Retrieval and submodular selection.}
RAG, MMR, and submodular context selection choose evidence after a query is known \cite{lewis2020rag,carbonell1998mmr,badanidiyuru2014,kumari2024bumblebee}. \method\ instead evaluates persistent write-time memory before future queries are known. Concave-over-modular objectives and knapsack submodular optimization are standard tools \cite{nemhauser1978,sviridenko2004,huang2017knapsack}; we use them to define and certify a benchmarkable semantic surrogate.

\paragraph{Dataset and evaluation protocols.}
Several long-memory datasets emphasize final answers after long histories. \method\ targets a different but compatible layer: whether the persistent store contains the needed evidence before the query-time pipeline runs. Each ratio is reproducible from a candidate table, cost vector, coverage matrix, evidence weights, and budget, making write-time failures auditable rather than hidden inside retrieval or prompt-state changes.

\section{Scope, Limitations, and Artifact}\label{sec:scope}

\method's scores are package-conditional by design: the candidate table, cost model, evidence schema, and budget define the evaluated object. This supplies a write-time diagnostic rather than replacing end-to-end assistant evaluation. If a generator omits the right memory, the package optimum measures the best store within that table; if a writer extracts useful memories but fails to retain the right subset, union and upper-pruned scores expose the selection bottleneck. Costs and one-representation groups are explicit package parameters; on Natural-87, byte-overhead costs and a $k=2$ relaxation change ratios at the percentage-points level rather than changing qualitative rankings.

Natural-87 is deliberately small, support-sliced, and model-adjudicated, with 13 ambiguous examples rejected from the stricter adjudication pass and human audits used to expose borderline entailment cases. Support slices isolate write-time memory before full-haystack expansion adds retrieval and session-selection confounds. The package-oracle reference writer uses exact coverage marginals; deployed systems need estimated utilities, learned generators, compaction, or native budget-aware admission.

\clearpage
\bibliographystyle{plain}
\bibliography{references}

\clearpage
\appendix

\section{Deferred Proofs}\label{app:proofs}

\begin{proof}[Proof of \Cref{thm:coverage}]
Normalization follows from $h_r(0)=0$. For monotonicity, adding any element $u$ increases each inner sum by $a_{ur}\geq0$, and $h_r$ is nondecreasing. For submodularity, let $S\subseteq T\subseteq\ground$ and $u\notin T$. Define $z_r(S)=\sum_{v\in S}a_{vr}$. Since $z_r(S)\leq z_r(T)$ and $h_r$ is concave, the increment $h_r(z+a_{ur})-h_r(z)$ is nonincreasing in $z$ for every $a_{ur}\geq0$. Multiplying by $w_r\geq0$ and summing over $r$ gives $\Delta(u\mid S)\geq\Delta(u\mid T)$.
\end{proof}

\begin{proof}[Proof of \Cref{prop:density-bad}]
Let $\budget=2$ and choose $\delta=\min\{\eta/2,1/4\}$. One experience has candidates $a,b$ with $c_a=\delta$, $F(\{a\})=2\delta$, $c_b=1$, and $F(\{b\})=1$. Candidate $a$ has density $2$, candidate $b$ has density $1$, so max-density chooses $a$ while OPT chooses $b$. The ratio is $2\delta\leq\eta$.
\end{proof}

\begin{proof}[Proof of \Cref{thm:gvt}]
Fix such a threshold. If the final cost $C(X)\geq\budget/2$, every accepted element had marginal density at least $\lambda$, so $F(X)\geq\lambda C(X)\geq(1-\varepsilon)\opt/4$. Otherwise $C(X)<\budget/2$. Let $O$ be optimal. By submodularity, $\opt\leq F(X)+\sum_{u\in O\setminus X}\Delta(u\mid X)$. Any missed optimal element was budget-feasible at its arrival. If it was below threshold, its marginal charge is at most $\lambda c_u$, and the total below-threshold charge is at most $\lambda\budget\leq\opt/2$. If it was threshold-admissible, the algorithm selected another candidate from the same group with at least as large raw marginal; these charges inject into accepted arrivals and sum to at most $F(X)$. Hence $\opt\leq2F(X)+\opt/2$, so $F(X)\geq\opt/4$.
\end{proof}

\section{Additional Experimental Details}\label{app:details}

This appendix records diagnostics omitted from the main story to keep the core paper within the 9-page main-body budget. These results are useful for reproducibility and interpretation, but the main claims rely on exact package optima and denominator-matched exported-store scoring.

\subsection{Controlled Package Generator}

\method-Small samples a compact hidden event graph, exposes an ordered stream of experiences, and generates finite representation choices for each experience. Candidate types include raw span, atomic fact, entity summary, tombstone, and compound update. Costs are measured in normalized word-equivalent units. Evidence units correspond to future-query support requirements, including positive units for validity state: a tombstone is useful when a future query requires knowing that an older fact has been invalidated, superseded, deleted, or should trigger abstention.

The benchmark reports ratios only when the denominator is certified. For the 500-seed exact-small run, each budgeted package is solved exactly with the pure-Python branch-and-bound solver. A separate 1,200-row audit verifies equality with the MILP formulation for the clipped-coverage objective. The no-tombstone OPT ablation removes tombstone and compound-update candidates before solving the same exact problem; it is therefore a representation-frontier comparison, not an algorithmic failure mode.

\begin{figure}[t]
\centering
\includegraphics[width=0.94\linewidth]{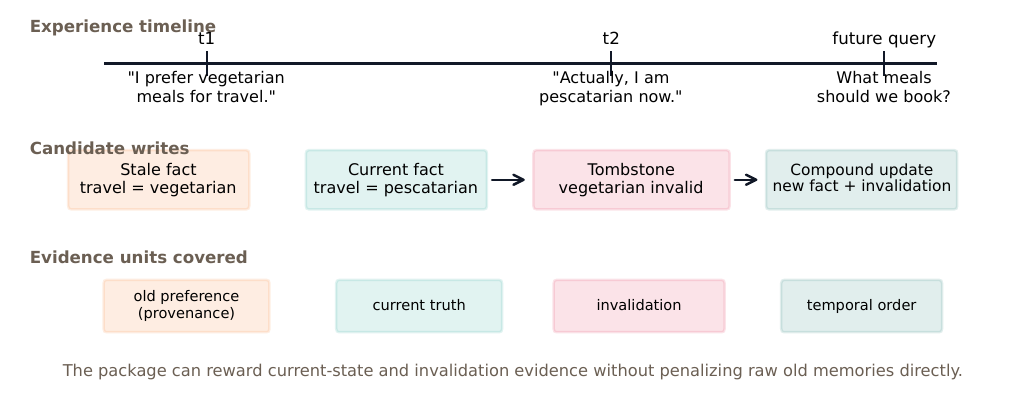}
\caption{Illustrative validity-state package. A stale fact, current fact, tombstone, and compound update are distinct candidate writes with different costs and evidence coverage. The package optimum may keep at most one candidate from an experience group, so representation choice is non-trivial rather than a union of all useful strings. The objective rewards current truth, supersession, deletion, and abstention evidence so stale-fact avoidance can be scored at write time.}
\label{fig:tombstone-timeline}
\end{figure}

\subsection{Reader Failure Audit}

The reader failure audit decomposes the LongMemEval-S focus diagnostic after the memory stores have been written. It is not a new denominator: it shows how often downstream errors come from missing retrieved evidence, abstention despite support, reader errors, unsupported answers, or scoring/parse uncertainty. This keeps the main package-ratio claim separate from retrieval and reader behavior.

\begin{figure}[t]
\centering
\includegraphics[width=0.92\linewidth]{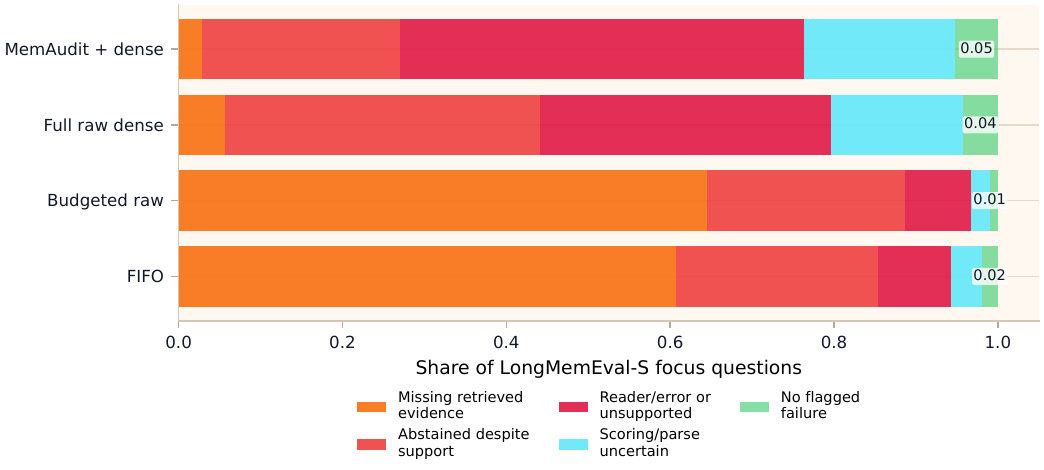}
\caption{LongMemEval-S focus reader failure audit. Bars decompose retrieval/reader outcomes after write-time memory selection; the green segment marks examples with no flagged failure in this audit.}
\label{fig:conditional-failure}
\end{figure}

\subsection{Natural Package Construction}

The Natural-200 package is built from support slices rather than full conversation histories. For each example, the candidate-generation stage receives source sessions and emits candidate memories with estimated costs. A separate annotation stage maps evidence units to future-query requirements and scores candidate coverage. We intentionally separate these stages so the writer does not see the query label during memory generation.

The annotation results should be read as model-adjudicated dataset construction with explicit audit hooks. The secondary 50-example audit surfaced ambiguous required-unit cases, which is why the main paper reports the stricter 87-example adjudicated subset for system diagnostics. The most common error/ambiguity categories were: required units that were too broad or could be decomposed several ways; summary candidates that implied but did not explicitly state an evidence unit; graph-edge memories whose relation label was plausible but too coarse; temporal or validity updates where current truth depended on ordering; and model-zero cells where humans judged a paraphrase sufficient. For the released human coverage audit, two project annotators familiar with the schema independently label all model-positive coverage cells on the 87-example subset plus a stratified 20\% sample of model-zero cells. We compute unweighted Cohen's $\kappa$ on the full stratified audit frame, rather than reweighting to the package's natural positive/negative balance. Human--human disagreements are concentrated in sampled model-zero cells ($116/162$) and in summary/graph-edge/atomic-fact boundaries; human--Gemini disagreements show the same pattern ($81/119$ from sampled zero cells), indicating that the main ambiguity is borderline entailment rather than schema failure. The human-edited seed set validates schema usability and remains in the artifact as prompt/schema development data.

A Claude Sonnet 4.5 cross-model audit gives a useful boundary on this result. Re-adjudicating the package with Sonnet preserves package-writer rankings exactly at $B=30,60,100$, but produces lower cell-level agreement than Gemini--human (\Cref{tab:coverage-audit}); we therefore treat Sonnet as a prompt/model-sensitivity stress test rather than an independent replacement annotator.

\subsection{External System Exports}

Mem0, A-Mem, and Letta are scored through exported written memories under a finite union denominator. The system memories are not forced into the package candidate schema; instead, they are added as new union candidates with their own costs and adjudicated coverage rows. This permits fair scoring of system-specific writes while preserving the exact denominator.

The main-system diagnostic reports post-hoc budget pruning because the public systems are not natively configured as package-budgeted writers. Recency pruning is a deployable baseline, salience pruning is a stronger heuristic, and upper pruning is an analysis-only bound over the system's own exported memories. A high upper-pruned bound with a low salience-pruned score indicates that extraction produced useful memories but budget-aware selection failed.

\section{Cost and Constraint Sensitivity}\label{app:sensitivity}

We audit two package-design choices on the 87-example adjudicated natural subset. First, we relax the partition constraint from one candidate per experience to two candidates per experience. Second, we replace the default word-equivalent cost with a per-record-overhead rule $c(u)=8+\lceil\mathrm{bytes}(u)/24\rceil$. All rows below are deterministic recomputations from cached package/export artifacts and make no API calls.

\begin{table}[h]
\centering
\caption{Package sensitivity on Natural-87. The $k=2$ relaxation changes exact OPT modestly, and the byte-overhead cost increases absolute OPT because short atomic memories remain cheap relative to the tested budgets.}
\label{tab:sensitivity-package}
\begin{tabular}{rrrrrr}
\toprule
$B$ & OPT $k=1$ & OPT $k=2$ & $k=2/k=1$ & byte-overhead OPT & byte/default \\
\midrule
30 & 0.931 & 0.943 & 1.013 & 1.310 & 1.322 \\
60 & 1.356 & 1.374 & 1.027 & 1.483 & 1.107 \\
100 & 1.489 & 1.540 & 1.060 & 1.632 & 1.131 \\
\bottomrule
\end{tabular}
\end{table}

\Cref{tab:sensitivity-systems} recomputes exported-system pruning at $B=100$ against the package-candidate denominator under both cost rules. The exact union-denominator results in the main text remain the primary system diagnostic; this table isolates whether the small-budget ranking is mostly a word-count artifact. The qualitative ordering is stable: Letta salience remains strongest among non-oracle exported selections, Mem0 salience trails, and A-Mem full-store remains infeasible at this budget. A-Mem's full notes are too large under both accounting rules: the mean minimum full-note cost is $1767.9$ words and $463.5$ byte-overhead units, far above $B=100$.

\begin{table}[h]
\centering
\caption{Exported-system cost sensitivity at $B=100$, reported as selected value divided by package-candidate OPT under each cost rule. Upper rows use hidden coverage and are analysis-only.}
\label{tab:sensitivity-systems}
\begin{tabular}{llr}
\toprule
Cost rule & Method & Ratio to package OPT \\
\midrule
word & Letta salience & 0.899 \\
word & Mem0 salience & 0.489 \\
word & A-Mem metadata recency & 0.190 \\
word & A-Mem full native & 0.000 \\
\midrule
byte-overhead & Letta salience & 0.800 \\
byte-overhead & Mem0 salience & 0.442 \\
byte-overhead & A-Mem metadata recency & 0.266 \\
byte-overhead & A-Mem full native & 0.000 \\
\bottomrule
\end{tabular}
\end{table}

\section{LongMemEval Transfer Diagnostics}\label{app:longmemeval}

We also ran a LongMemEval-S retrieval/reader diagnostic. These runs are not package-oracle ratios because the full benchmark does not provide a finite candidate package with exact coverage rows. They are included to check whether package-style writing can reduce retrieval cost and improve evidence use in a more natural long-memory task.

\begin{figure}[t]
\centering
\includegraphics[width=0.96\linewidth]{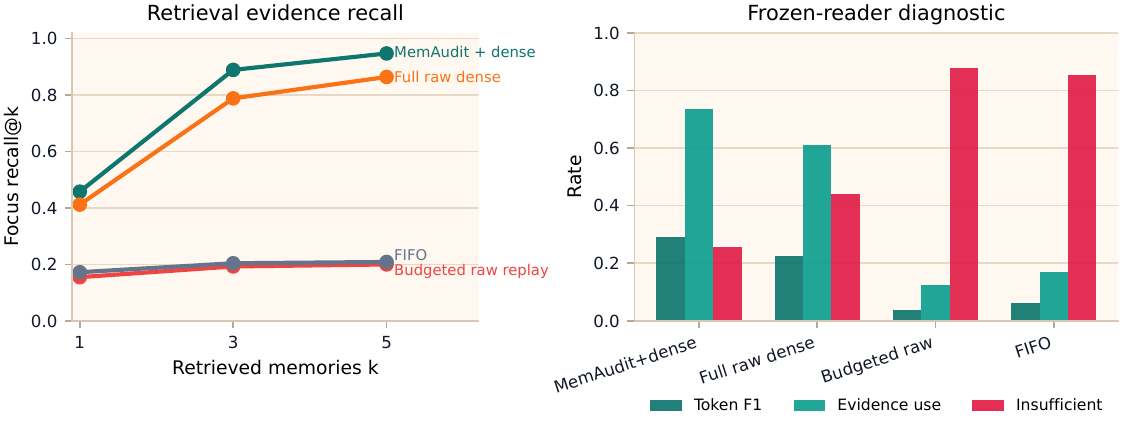}
\caption{LongMemEval-S transfer diagnostics. These are evidence-use and reader diagnostics, not exact package-oracle ratio results. Exact-match gains are not the basis of the main claim.}
\label{fig:longmemeval}
\end{figure}

The useful signal is narrower than end-to-end QA superiority. The \method-style store improved retrieval compactness and token-F1/evidence-use behavior relative to full raw-store dense retrieval, while exact-match changes were not the primary evidence. This is why the main paper uses LongMemEval only as a transfer diagnostic and does not claim significant answer-accuracy gains.

\section{Reproducibility Notes}\label{app:repro}

{\raggedright
The artifact contains deterministic exact-package runs, natural package exports, model-adjudication reports, and public-system scoring scripts. The safest reproduction path is to run the exact-small suite against certified OPT, audit a random subset with the MILP solver, re-score the 87-example adjudicated natural subset under package and union denominators, and recompute the public-system pruning diagnostics from exported Mem0, A-Mem, and Letta stores.
API-backed natural package construction caches prompts, completions, token counts, model versions, and adjudication decisions. Natural package candidate generation used OpenRouter with \texttt{google/gemini-3.1-flash-lite-preview}, temperature $0$, JSON response format, and a default 1400-token completion cap. Exported-system coverage/salience adjudication used cached Gemini Flash/Flash-Lite calls with temperature $0$ and JSON response format; the Letta coverage and salience caps were 1800 and 4000 tokens respectively, while the Mem0 writer configuration used temperature $0$ and a 700-token cap. The cross-model stress audit used OpenRouter \texttt{anthropic/claude-sonnet-4.5} with temperature $0$, JSON response format, cached package-level adjudications, and a separate cached 1{,}975-cell binary coverage pass. All cached calls include prompt hashes and token/cost metadata. The exact package-oracle denominator itself is deterministic once a package has been constructed.
\par}

\end{document}